\documentclass{article}

\pdfoutput=1

\usepackage{microtype}
\usepackage{graphicx}
\usepackage{subfigure}
\usepackage{booktabs}

\usepackage{verbatim}
\usepackage{amsmath}
\usepackage{amssymb}
\usepackage{amsfonts}
\usepackage{relsize}
\usepackage[hyphens]{url}
\usepackage[hidelinks]{hyperref}
\hypersetup{breaklinks=true}
\usepackage[accepted]{icml2019}

\newcommand{\E}{\mathrm{E}}
\renewcommand{\P}{\mathrm{P}}
\newcommand{\mbf}{\mathbf}

\newtheorem{theorem}{Theorem}
\newtheorem{lemma}{Lemma}
\newtheorem{corollary}{Corollary}

\newtheorem{definition}{Definition}
\newtheorem{remark}{Remark}
\newenvironment{proof}[1][]{\paragraph{Proof #1:}}{\hfill$\square$}

\begin{document}

\twocolumn[
\icmltitle{Making Convex Loss Functions Robust to Outliers using $e$-Exponentiated Transformation}

\icmlsetsymbol{equal}{}

\begin{icmlauthorlist}
\icmlauthor{Suvadeep Hajra}{equal,iitb}
\end{icmlauthorlist}

\icmlaffiliation{iitb}{Department of Computer Science and Engineering, Indian Institute of Technology Bombay, Mumbai, India}

\icmlcorrespondingauthor{Suvadeep Hajra}{suvadeep.hajra@gmail.com}

\icmlkeywords{loss function, robust classifier, label noise, deep neural network, generalization error bound, nonconvex loss function}

\vskip 0.3in
]

\printAffiliations{}

\begin{abstract}
In this paper, we propose a novel {\em $e$-exponentiated} transformation, $0 \le e<1$, for loss functions. When the transformation is applied to a convex loss function, the transformed loss function become more robust to outliers.
Using a novel generalization error bound, we have theoretically shown that the transformed loss function has a tighter bound for datasets corrupted by outliers. Our empirical observation shows that the accuracy obtained using the transformed loss function can be significantly better than the same obtained using the original loss function and comparable to that obtained by some other state of the art methods in the presence of label noise.
\end{abstract}

\section{Introduction}
\label{intro}
Convex loss functions are widely used in machine learning as their usage lead to convex optimization problem in a single layer neural network or in a kernel method. That, in turn, provides the theoretical guarantee of getting a globally optimum solution efficiently. However, many earlier studies have pointed out that convex loss functions are not robust to outliers \cite{LongS08,LongS10,DingV10,ManwaniS13,RooyenMW15,GhoshMS15}. Indeed a convex loss imposes a penalty which grows at least linearly with the negative margin for a wrongly classified example, thus making the classification hyperplane greatly impacted by the outliers. Consequently, nonconvex loss functions have been widely studied as a robust alternative to convex loss function \cite{Masnadi-ShiraziV08,LongS10,DingV10,DenchevDVN12,ManwaniS13,GhoshMS15}. 

In this paper, we propose $e$-exponentiated transformation for loss function to make a convex loss functions more robust to outliers. Given a convex loss function $l(\hat{y}, y)$, we define it's $e$-exponentiated transformation to be $l^{e,c}(\hat{y},y)=l(\sigma^{e,c}(\hat{y}),y)$ for $0 \le e< 1$ and some real positive constant $c$ where $\sigma^{e,c}(\hat{y})$ is given by
\begin{align}
 \sigma^{e,c}(\hat{y}) = \left\{
                          \begin{array}{ll}
                            sgn(\hat{y})|\hat{y}|^e     & \text{ if }|\hat{y}|\ge c\\
                            c^{e-1}\hat{y} & \text{ otherwise}
                          \end{array}\right. \label{eq:sigma}
\end{align}
with $|\hat{y}|$ denoting the absolute value of $\hat{y}\in\mathbb{R}$ and the sign function $sgn(\hat{y})$ defined to be equal to $1$ for $\hat{y}\ge 0$, $-1$ otherwise. For a differentiable convex loss function $l(\cdot, \cdot)$, its $e$-exponentiated transformation $l^{e,c}(\cdot,\cdot)$ is differentiable everywhere except at $\hat{y}\in\{-c, c\}$. Thus, a gradient based optimization algorithm can be used for empirical risk minimization with $e$-exponentiated loss function. Moreover, an $e$-exponentiated loss function $l(\hat{y},y)$ is more robust to outliers than the corresponding convex loss function $l(\hat{y},y)$ as the slope $|\frac{d}{d\hat{y}}l^{e,c}(\hat{y},y)|=e|\hat{y}|^{e-1}|\frac{d}{d\sigma^{e,c}(\hat{y})}l(\sigma^{e,c}(\hat{y}),y)|< |\frac{d}{d\hat{y}}l(\hat{y},y)|$ for $\hat{y}< -1$ (please refer to Figure~\ref{fig:loss}).
\begin{figure}[h!]
\vskip 0.2in
\begin{center}
\centerline{\includegraphics[width=\columnwidth]{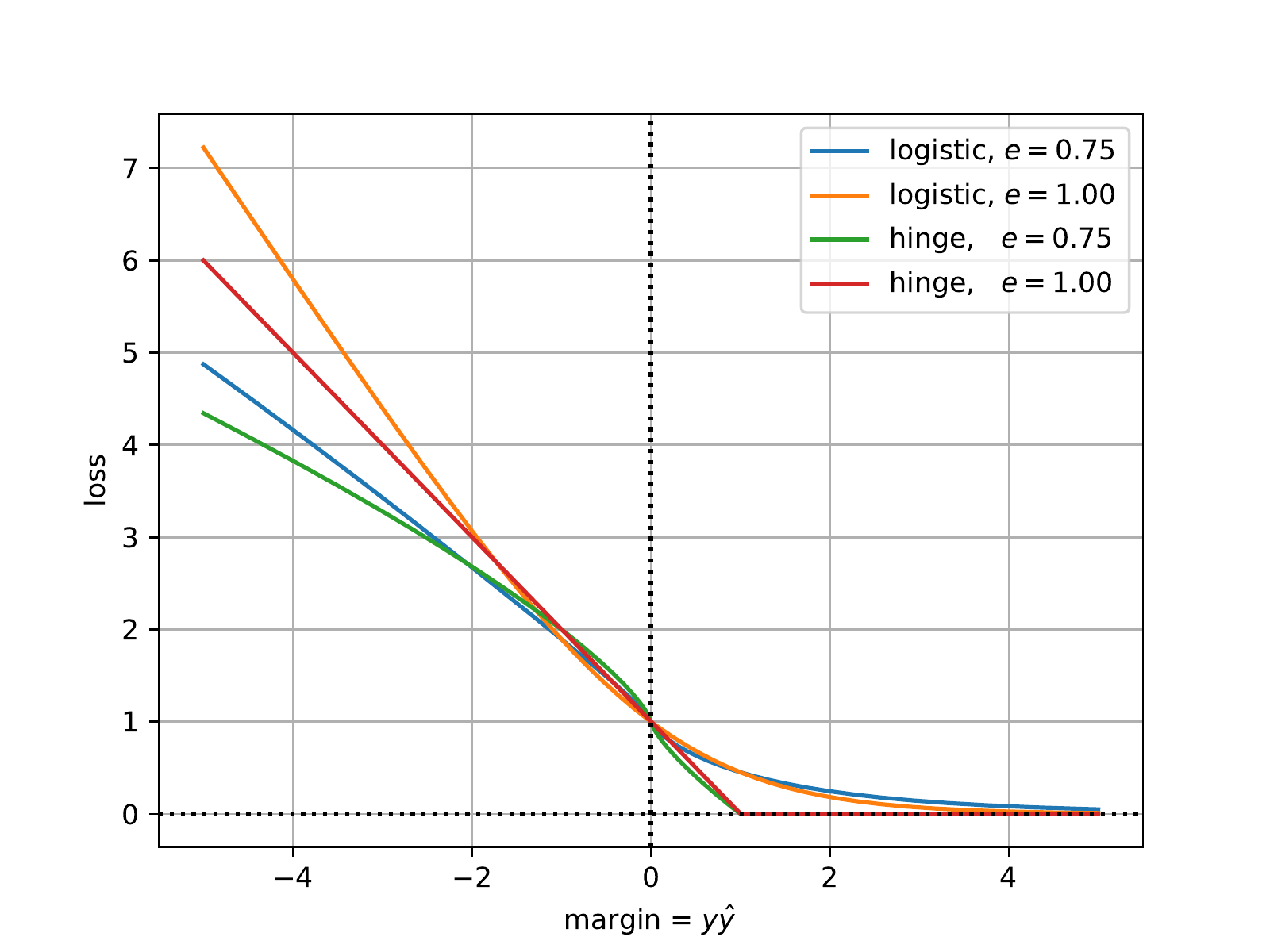}}
\caption{$e$-exponentiated transformation of logistic and hinge loss. $e=1.0$ implies the original loss. We have used $c=0.005$ in all the plots.}
\label{fig:loss}
\end{center}
\vskip -0.2in
\end{figure}

Additionally, by introducing a novel generalization error bound, we show that the bound for an $e$-exponentiated loss function can be tighter than the corresponding convex loss function. Unlike existing generalization error bounds \cite{rosasco04} which strongly depends on the Lipschitz constant of a loss function, our derived bound depends on the Lipschitz constant only weakly. Consequently, even having a larger Lipschitz constant for an $e$-exponentiated loss function compared to the corresponding convex loss function, the bound can be tighter.

In summary, the contributions of the paper are as follows:
\begin{enumerate}
 \item In this paper, we propose an $e$-exponentiated transformation of convex loss function. The proposed transformation can make a convex loss function more robust to outliers.
 \item Using a novel generalization error bound, we show that the bound for an $e$-exponentiated loss function can be tighter than the corresponding convex loss function. Our derived bound only weakly depends on the Lipschitz constant of a loss function. Consequently, our bound for a loss function can be tighter in spite of having a larger Lipschitz constant.
 \item We have empirically verified the accuracy obtained by our proposed $e$-exponentiated loss functions on several datasets. The results show that we can get significantly better accuracies using the $e$-exponentiated loss function than that obtained by the corresponding convex loss and comparable accuracies to that obtained by some other state of the art methods in the presence of label noise.
\end{enumerate}

The organization of the work is as follows. In Section~\ref{sec:emp_risk_min}, we have formally introduced the empirical risk minimization problem. Section~\ref{sec:error_bound} derives a novel generalization error bound. Using the bound, we have also shown that the bound can be tighter for $e$-exponentiated loss function. In Section~\ref{sec:exp_results}, we have shown our experimental result. Finally, Section~\ref{sec:con} concludes the work.

\section{Empirical Risk Minimization Using $e$-Exponentiated Loss}\label{sec:emp_risk_min}
We consider the empirical risk minimization of a linear classifier with $e$-exponentiated loss function for a binary classification problem. Given a convex loss function $l(\cdot, \cdot)$, the empirical risk minimization of a linear classifier is given by:
\begin{small}
\begin{align}
 \hat{\mathrm{R}}_l(\mbf{w};\mathcal{D}) = \frac{1}{N}\sum_{i=1}^N l(\hat{y}_i, y_i) = \frac{1}{N}\sum_{i=1}^Nl(\mbf{w}^T\phi(\mbf{x}_i), y_i) \label{eq:emp_risk}
\end{align}
\end{small}
where $\mathcal{D}=\{(\mbf{x}_i, y_i)\}_{i=1}^N$ is the training set, $\phi(\mbf{x})\in \mathbb{R}^d$ is the feature representation of the sample $\mbf{x}$ and the target $y_i$s takes a value from $\{1, -1\}$ for $i\in \{1, \cdots, N\}$. The corresponding empirical risk with $e$-exponentiated loss $l^{e,c}(\cdot, \cdot)$ is given by:
\begin{footnotesize}
\begin{align}
 \hat{\mathrm{R}}_{l^{e,c}}(\mbf{w};\mathcal{D}) &= \frac{1}{N}\sum_{i=1}^N l^{e,c}(\hat{y}_i, y_i) = \frac{1}{N}\sum_{i=1}^N l(\sigma^{e,c}(\hat{y}_i), y_i) \nonumber\\
 &= \frac{1}{N}\sum_{i=1}^Nl(\sigma^{e,c}(\mbf{w}^T\phi(\mbf{x}_i)), y_i)
\end{align}
\end{footnotesize}
where $e\in [0, 1)$, $c> 0$ and $\sigma^{e,c}(\cdot)$ is as defined in Eq.~\eqref{eq:sigma}. In the rest of the paper, we will ignore the second argument of $\hat{\mathrm{R}}_l(\mbf{w};\mathcal{D})$ and $\hat{\mathrm{R}}_{l^{e,c}}(\mbf{w};\mathcal{D})$ whenever $\mathcal{D}$ can be inferred from the context.

\section{Generalization Error Bounds of Empirical Risk Minimization with $e$-Exponentiated Loss}\label{sec:error_bound}
In this section, we present an upper bound for the generalization error incurred by an $e$-exponentiated loss function. Towards this end, we first propose a novel method for estimating the upper bound. Our introduced method of generalization error bound captures the average behaviour of a loss function as opposed to other existing methods \cite{rosasco04} which captures the worst case behaviour. More particularly, our method is more suitable for analysing nonconvex problems where the risk function is smooth in most of the regions but contains some very low probable high gradient regions. Consequently, our bound shows an weak dependence on the Lipschitz constant of the loss functions as opposed to other existing methods \cite{rosasco04} which depend on the Lipschitz constant monotonically. Finally, applying the derived bound, we show that empirical risk minimization with $e$-exponentiated loss function can have tighter generalization error bound than that can be obtained using the corresponding convex loss function.

\subsection{Upper Bound for the Generalization Error}
The gradient of an $e$-exponentiated loss function can be very large (in the order of $c^{e-1}\times L_l$ where $L_l$ is the Lipschitz constant of the corresponding convex loss) making the Lipschitz constant of the transformed loss very large for $c\ll 1$. On the other hand, the existing generalization error bound gets loose as the Lipschitz constant gets larger. To overcome this issue, we propose a novel bound for the same. Our bound is based on the work of \cite{rosasco04}. Before stating our bound, let us introduce certain notations and definitions.

\begin{definition}
 A function $f:A\mapsto\mathbb{R}$, $\mathcal{A}\subseteq\mathbb{R}^n$ is said to be $L_f$-Lipschitz continuous, $L_f>0$, if 
 \begin{align}
  |f(\mbf{a})-f(\mbf{b})| \le L_f||\mbf{a}-\mbf{b}||_2
 \end{align}
for every $\mbf{a}, \mbf{b}\in \mathcal{A}$.
\end{definition}

\begin{definition}
 A function $f:A\mapsto\mathbb{R}$, $\mathcal{A}\subseteq\mathbb{R}^n$, is said to be {\em Lipschitz in the small} continuous, if there exists $\epsilon>0$ and $L_f(\epsilon)>0$ such that
 \begin{small}
 \begin{align}
  ||\mbf{a}-\mbf{b}||_2 \le \epsilon \;\text{implies }\; |f(\mbf{a})-f(\mbf{b})|\le L_f(\epsilon)||\mbf{a}-\mbf{b}||_2\label{eq:def_ls}
 \end{align}
 \end{small}
for every $\mbf{a}, \mbf{b}\in \mathcal{A}$.
\end{definition}
\noindent
Note that, in general, whenever a function $f(\mbf{x})$ is continuous and differentiable, $L_f\ge L_f(\epsilon) \ge \text{sup}_{\mbf{x}}\; |f'(\mbf{x})|=L_f$ for all $\epsilon>0$ where $f'(\mbf{x})$ is the gradient of $f(\mbf{x})$ at $\mbf{x}$. However, this might not be true when the function $f(\mbf{x})$ also depends on the distribution of the input $\mbf{x}$.

With the above definitions, we state our generalization error bound in the next theorem. Note that since a close ball in $\mathbb{R}^d$ defined as $\mathcal{W}_M\triangleq\{\mbf{w}\in \mathbb{R}^d|\;||\mbf{w}||_2\le M\}$ is a compact set, we can cover the set by taking union of a finite number of balls of radius $\epsilon$ for any $\epsilon>0$. Let us denote the covering number of $\mathcal{W}_M$ by $C(\epsilon)$. Also, we define the expected risk corresponding to the empirical risk given by Eq.~\eqref{eq:emp_risk}
\begin{align}
 \mathrm{R}_l(\mbf{w}) = \E_{\mbf{x},y}[l(\mbf{w}^T\phi(\mbf{x}), y)]
\end{align}
where $\E_{\mbf{x},y}[\cdot]$ denotes expectation over the joint distribution of $\mbf{x}$ and $y$. Also note that so far we have used the notation $l(\cdot, \cdot)$ to represent a convex loss function. However, in this section, we use the notation to represent any arbitrary loss function. With the above definitions and notations, we state our generalization error bound in the following theorem.

\begin{theorem}\label{thm:2}
 Let $\mathcal{D}_N = (\mbf{x}_i, y_i)_{i=1}^N$ such that $\phi(\mbf{x})_i\in\{\phi(\mbf{x})\in\mathbb{R}^d\;|\;||\phi(\mbf{x})||_2\le 1\}$, and $y_i\in\{-1, +1\}$. Let $\mbf{w}\in \mathcal{W}_M\triangleq\{\mbf{w}\in \mathbb{R}^d|\;||\mbf{w}||_2\le M\}$ with $M\ge 1$. Let the loss function $l(\cdot, \cdot)$ is $L_l$-Lipschitz continuous. Set $B=L_{\mathrm{R}_l}(M)M+C_l$ where $L_{\mathrm{R}_l}(\epsilon)$ is as defined in Eq.~\eqref{eq:def_ls} and $C_l>0$ such that $C_l\ge l(0, y)$ for $y\in\{-1, +1\}$. Then for all $\epsilon>0$, we have
 \begin{smaller}
 \begin{align}
  &\P\left(\left\{\mathcal{D}_N\mathlarger{\mathlarger{|}}\underset{\mbf{w}\in \mathcal{W}_M}{sup}|\mathrm{R}_{l}(\mbf{w})-\hat{\mathrm{R}}_{l}(\mbf{w};\mathcal{D}_N)|\le \epsilon+\frac{L_l\epsilon^2}{2B}\right\}\right) \nonumber\\
  &\ge 1 - 2\left(C\left(\frac{\epsilon}{4L_{\mathrm{R}_l}(\epsilon')}\right)+1\right)exp\left(-\frac{N\epsilon^2}{8B^2}\right).\label{eq:thm_2}
 \end{align}
 \end{smaller}
 where $\epsilon'>0$ such that $\epsilon'\ge \text{min }\left\{\epsilon,\frac{\epsilon}{4L_{\mathrm{R}_l}(\epsilon')}\right\}$. (and this always exists).
\end{theorem}
Proof of Theorem~\ref{thm:2} has been skipped to Appendix~\ref{app:2}. To compare our result with the previous result, we state the result of \cite{rosasco04} in the next theorem:
\begin{theorem}\cite{rosasco04}\label{thm:3}
   Let $\mathcal{D}_N$, $M$, $\mathcal{W}_M$, $L_l$ and $C_l$ are as defined in Theorem~\ref{thm:2}. Set $B=L_lM+C_l$. Then for all $\epsilon>0$, we have
 \begin{align}
  &\P\left(\left\{\mathcal{D}_N\,\mathlarger{\mathlarger{|}}\;\underset{\bf{w}\in \mathcal{W}_M}{sup}\left|\mathrm{R}_l(\mbf{w})-\hat{\mathrm{R}}_l(\mbf{w};\mathcal{D}_N)\right|\le \epsilon\right\}\right) \nonumber\\
  &\ge 1 - 2C\left(\frac{\epsilon}{4L_l}\right)exp\left(-\frac{N\epsilon^2}{8B^2}\right).\label{eq:thm_3}
 \end{align}
\end{theorem}

\begin{remark}
 The confidence bound in the RHS of Eq.~\eqref{eq:thm_3} involves $L_l$, the Lipschitz constant of the loss function. Thus, the bound is a monotonically decreasing function of $L_l$ i.e. it gets worse as $L_l$ gets larger. On the other hand, the confidence bound of Eq.~\eqref{eq:thm_2} no more involve the Lipschitz constant of the loss function $L_l$. Instead, it involves $L_{\mathrm{R}_l}(\epsilon)$ which can be reasonably small even when $L_l$ is very large.
 \end{remark}

\begin{remark}
 By comparing Eq.\eqref{eq:thm_2} and \eqref{eq:thm_3}, we see that there are two main differences. First, in LHS of Eq.~\eqref{eq:thm_2}, $\epsilon$ has been replaced by a slightly larger quantity $\epsilon+L_l\epsilon^2/2B$. Since we generally take $\epsilon \ll 1$ and $B\ge 1$, $L_l\epsilon^2/2B$ can be a negligible quantity even for reasonably large $L_l$. Thus, it does not compromise the error bound significantly. Secondly, in RHS Eq.~\eqref{eq:thm_2}, $C(\epsilon/4L_l)$ has been replaced by $C(\epsilon/4L_{\mathrm{R}_l}(\epsilon'))+1$. Since for $x\ll 1$, the covering number $C(x)\gg 1$, Eq.~\eqref{eq:thm_2} also does not compromise the confidence probability significantly. Moreover, if $L_{\mathrm{R}_l}(\epsilon')$ is reasonably smaller than  $L_l$, the confidence bound given by Eq.~\eqref{eq:thm_2} can be significantly better than that given by Eq.~\eqref{eq:thm_3}.
\end{remark}

\begin{remark}
 For the nonconvex problem where the risk is smooth on most of the regions in its domain but has very high gradient on some very low probable regions, the bound given by Theorem~\ref{thm:3} can be very loose as the corresponding Lipschitz constant can be very large. However, Theorem~\ref{thm:2} can still provides a tight bound under proper distributional assumption. Thus, Theorem~\ref{thm:2} is better suitable for analysing nonconvex problems.
\end{remark}

\subsection{Comparison of Generalization Error Bound}
From Theorem~\ref{thm:2}, we see that when $L_l\epsilon/2B\ll 1$, the generalization error bound is a monotonically decreasing function of $L_{\mathrm{R}_l}(\epsilon)$ where $l(\cdot, \cdot)$ is the loss function used in the empirical risk minimization. Thus, to compare the generalization error bound of an $e$-exponentiated loss function with that of the corresponding convex loss function, we compare $L_{\mathrm{R}_l}(\epsilon)$ with $L_{\mathrm{R}_{l^{e,c}}}(\epsilon)$ where $l(\cdot, \cdot)$ is a convex loss function and $l^{e,c}(\cdot, \cdot)$ is its $e$-exponentiated transformation. Since $L_{\mathrm{R}_{l}}(\epsilon)$ depends on the distribution $\mbf{x}$ and $y$, we assume that the margin $y\hat{y}=y\mbf{w}^T\phi(\mbf{x})$ follows an uniform distribution. Moreover, since by our previous assumptions, $||\phi(\mbf{x})||_2\le 1$ and $||\mbf{w}||_2\le M$, $|y\hat{y}|\le M$. Note that in this case,
\begin{align*}
L_{\mathrm{R}_{l^{e,c}}}(\epsilon)=L_{\mathrm{R}_{l^{e,c}}}(M)=\underset{||\mbf{w}||_2\le M}{\text{sup}}\;||\frac{d}{d\mbf{w}}\mathrm{R}_{l^{e,c}} ||_2=L_{\mathrm{R}_{l^{e,c}}}
\end{align*}
Thus, we compute an upper bound of $L_{\mathrm{R}_{l^{e,c}}}$ as
\begin{align}
 L_{\mathrm{R}_{l^{e,c}}} &= \underset{||\mbf{w}||_2\le M}{\text{sup}} \mathlarger{\mathlarger{||}}\frac{d}{d\mbf{w}} \E_{\mbf{x},y}[l(\sigma^{e,c}(\mbf{w}^T\phi(\mbf{x})), y)]\,\mathlarger{\mathlarger{||}}_2 \nonumber \\
 &= \underset{||\mbf{w}||_2\le M}{\text{sup}} \mathlarger{\mathlarger{||}}\E_{\mbf{x},y}\left[\frac{d}{d\mbf{w}}l(\sigma^{e,c}(\mbf{w}^T\phi(\mbf{x})), y)\right]\mathlarger{\mathlarger{||}}_2 \nonumber \\
 &\equiv \left|\E_{-M\le \delta\le M}\left[\frac{d}{d\delta}l(\sigma^{e,c}(\delta))\right]\right| \text{ where }\delta=\hat{y}y \label{eq:L_R}
\end{align}

The RHS of Eq.~\eqref{eq:L_R} can be shown to be less than $L_{\mathrm{R}_l}\equiv|\E_{-M\le\delta\le M}\left[\frac{d}{d\delta}l(\delta)\right]|$ for sufficiently large $M$ and convex loss function $l(\cdot, \cdot)$ with non-positive gradient. Note that most of the standard convex loss functions for classification have gradient which is non-positive.

In the next section, we show the experimental results using $e$-exponentiated loss functions.

\section{Experimental Results}\label{sec:exp_results}
To demonstrate the improvement obtained using $e$-exponentiated loss functions empirically, we show the results of two sets of experiments. In the first set of experiments, we have compared the accuracies obtained using $e$-exponentiated loss function with that obtained using the corresponding convex loss function on a subset of ImageNet dataset \cite{imagenet_cvpr09}. In the second set of experiments, we compared the $e$-exponentiated loss functions with other state of the art methods for noisy label learning on four datasets.

\subsection{Experiments on ImageNet Dataset}
To show the improvement in accuracies using the $e$-exponentiated loss functions over the corresponding convex loss functions, we have performed experiments on a subset of ImageNet dataset. Our collected subset of ImageNet dataset contains $511,544$ images of $1000$ labels. We have randomly splitted the dataset into training set of $400,000$, validation set of $50,000$ and test set of $61,544$ images. For the experiments, we have extracted pre-trained features of the images by passing them through the first five layers of a pre-trained AlexNet model \cite{KrizhevskySH12}. We have downloaded the pre-trained model from \cite{shelhamer} and use the code of \cite{kratzert} for extracting the pre-trained features. Note that there are only $223$ labels common in between our subset of ImageNet dataset and ImageNet LSVRC-2010 contest dataset on which the AlexNet model has been pre-trained.

For classification using the pre-trained features, we have used a three layer fully connected neural network with ReLU activation. We performed the experiments using the $e$-exponentiated softmax loss and logistic loss by varying $e = 1, 0.75 \text{ and }0.60$ and setting $c=0$. Note that $e=1$ gives us the original convex loss function. We set the dimension of the hidden layers to be $800$ and used Adam optimizer for optimization. To find the suitable value of initial learning rate and keep probability for the dropout, we performed cross-validation using the top-5 accuracy on the validation set. The top-1 and top-5 test accuracies of all the experiments are shown in Table~\ref{tb:imagenet_results}. The results shows that we have got a $3$ to $4$\% improvement in top-1 and top-5 accuracies for $e=0.6$ over $e=1.0$ for both softmax and logistic loss. For $e=0.75$, the accuracies obtained are in between the accuracies obtained by $e=0.60$ and $e=1.0$.

\begin{table}
 \begin{center}
  \begin{tabular}{|l|c|c|c|}
   \hline
   Loss function & $e$ & Top-1 & Top-5 \\
   \hline
   \hline
                 & $0.60$ & $\mbf{40.84}$ & $\mbf{67.01}$ \\
   \cline{2-4}
   Logistic      & $0.75$ & $39.12$ & $65.66$ \\
   \cline{2-4}
                 & $1.00$ & $36.95$ & $63.44$ \\
   \hline
   \hline
                 & $0.60$ & $\mbf{39.30}$ & $\mbf{67.01}$ \\
   \cline{2-4}
   Softmax       & $0.75$ & $36.00$ & $63.52$ \\
   \cline{2-4}
                 & $1.00$ & $35.31$ & $62.88$ \\
   \hline
  \end{tabular}
 \end{center}
 \caption{Top-1 and Top-5 accuracies obtained on subset of ImageNet dataset. We have used $e$-exponentiated logistic and softmax loss function. Experiments are performed using $e = 1, 0.75 \text{ and } 0.60$. Note that $e$-exponentiated loss function with $e=1$ gives us back the original convex loss function.\label{tb:imagenet_results}}
\end{table}

\subsection{Comparison with Other State-of-the-art Methods for Noisy Label Learning}
In this section, we compare the accuracies obtained using $e$-exponentiated loss function with other state-of-the-art methods by adding label noise on the training set. For the purpose, we have adopted the experimental setup of \cite{MaWHZEXWB18}.

\paragraph*{Experimental Setup} As in \cite{MaWHZEXWB18}, we performed the experiments by adding $0\%$, $20\%$, $40\%$ and $60\%$ symmetric label noise on four benchmark datasets: MNIST (\cite{lecun98}), SVHN (\cite{netzer11}), CIFAR-10 (\cite{krizhevsky09}) and CIFAR-100 (\cite{krizhevsky09}). For all the datasets, we have used the same model and optimization setup as used in \cite{MaWHZEXWB18}. Additionally, we have performed experiments using $e$-exponentiated softmax loss function with $c=0.005$ and varying $e=1.0, 0.75 \text{ and } 1.0$. As mentioned earlier, $e=1$ gives us back the corresponding softmax loss. Following Ma et al. in \cite{MaWHZEXWB18}, we have repeated the experiments five times and reported the mean accuracies.

\paragraph*{Baseline Methods} For the comparison purpose, we have used the baseline methods which have been used in \cite{MaWHZEXWB18}. For the shake of completeness, we briefly describe those:
\begin{description}
 \item[Forward \cite{PatriniRMNQ17}] Noisy labels are corrected by multiplying the network predictions with a label transition matrix.
 \item[Backward \cite{PatriniRMNQ17}] Noise labels are corrected by multiplying the loss by the inverse of a label transition matrix.
 \item[Boot-soft \cite{ReedLASER14}] Loss function is modified by replacing the target label by a convex combination of the target label and the network output.
 \item[Boot-hard \cite{ReedLASER14}] It is same as {\em Boot-soft} except that instead of directly using the class predictions in the convex combination, it converts the class prediction vector to a $\{0, 1\}$-vector by thresholding before using in the convex combination.
 \item[D2L \cite{MaWHZEXWB18}] It uses an adaptive loss function which exploits the differential behaviour of the deep representation subspace while a network is trained on noisy labels.
\end{description}

\paragraph*{Training with $e$-Exponentiated Loss function} We have found that for larger network, the rate of convergence using $e$-exponentiated loss function in the initial iterations are slow due to smaller magnitude of gradients. For a similar problem, Barron et al., in \cite{Barron19}, have used an ``annealing'' approach in which, at the beginning of the optimization, they start with a convex loss function and at each epoch they gradually make the loss function nonconvex by slowly tuning a hyper-parameter. However, in our experiments, we take a simpler approach. For the first $total\_epoch/10$ epochs, where $total\_epoch$ is the total number of epochs the model is trained, we trained the model by setting $e=1$. After $total\_epoch/10$ epochs, we switch the value of $e$ to our desired lower value. We take it as a future work to use a more sophisticated approaches like ``annealing'' in our experiments.

\paragraph*{Results} The results are shown in Table~\ref{tb:comp_results}.
\begin{table*}[!ht]
\begin{footnotesize}
\begin{center}
  \begin{tabular}{|l|c|c|c|c|c|c||c|c|c|}
   \hline
   \bf{Dataset} & \bf{Noise} & \bf{Forward} & \bf{Backward} & \bf{Boot-hard} & \bf{Boot-soft} & \bf{D2L} &\multicolumn{3}{|c|}{\bf{Softmax Crossentropy}}\\
   \cline{8-10}
          & \bf{Rate}  &         &          &           &           &     & $\mbf{e=1.00}$ & $\mbf{e=0.75}$ & $\mbf{e=0.60}$\\
   \hline
          & $0\%$   & $99.30$ & $99.23$ & $99.13$ & $99.20$ & $99.28$ & $99.28$  & $99.30$ & $99.30$ \\
   \cline{2-10}
   \bf{MNIST}& $20\%$  & $96.45$ & $90.12$ & $87.69$ & $88.50$ & $98.84$ & $88.29$        & $88.76$ & $89.16$ \\
   \cline{2-10}
          & $40\%$  & $94.90$ & $70.89$ & $69.49$ & $70.19$ & $98.49$ & $68.70$ & $69.18$ & $71.93$ \\
   \cline{2-10}
          & $60\%$  & $82.88$ & $52.83$ & $50.45$ & $46.04$ & $94.73$ & $46.12$ & $46.39$ & $49.23$ \\
   \hline
          & $0\%$   & $90.22$ & $90.16$ & $89.47$ & $89.26$ & $90.32$ & $91.09$ & $91.02$ & $91.07$ \\
   \cline{2-10}
   \bf{SVHN}& $20\%$  & $85.51$ & $79.61$ & $81.21$ & $79.26$ & $87.63$ & $78.99$ & $79.03$ & $78.28$ \\
    \cline{2-10}
          & $40\%$  & $79.09$ & $64.15$ & $63.25$ & $64.30$ & $82.68$ & $61.43$ & $61.15$ & $60.26$ \\
    \cline{2-10}
          & $60\%$  & $62.57$ & $53.14$ & $47.61$ & $39.21$ & $80.92$ & $39.17$ & $39.23$ & $38.73$ \\
    \hline
          & $0\%$   & $90.27$ & $89.03$ & $89.06$ & $89.46$ & $89.41$ & $90.33$ & $90.36$ & $90.17$ \\
    \cline{2-10}
   \bf{CIFAR-10}& $20\%$& $84.61$& $79.41$ & $81.19$ & $79.21$ & $85.13$ & $82.00$ & $82.94$ & $84.70$ \\
    \cline{2-10}
          & $40\%$  & $82.84$ & $74.69$ & $76.67$ & $73.81$ & $83.36$ & $75.60$ & $75.86$ & $78.62$ \\
    \cline{2-10}
          & $60\%$  & $72.41$ & $45.42$ & $70.57$ & $68.12$ & $72.84$ & $67.02$ & $68.36$ & $72.35$ \\
    \hline
          & $0\%$   & $68.54$ & $68.48$ & $68.31$ & $67.89$ & $68.60$ & $68.56$ & $68.34$ & $67.47$ \\
    \cline{2-10}
 \bf{CIFAR-100}& $20\%$ & $60.25$ & $58.74$ & $58.49$ & $57.32$ & $62.20$ & $59.84$ & $61.08$ & $61.96$ \\
    \cline{2-10}
          & $40\%$  & $51.27$ & $45.42$ & $44.41$ & $41.87$ & $52.01$ & $51.56$ & $53.05$ & $54.27$ \\
    \cline{2-10}
          & $60\%$  & $41.22$ & $34.49$ & $36.65$ & $32.29$ & $42.27$ & $38.71$ & $39.41$ & $39.56$ \\
    \hline
  \end{tabular}
\end{center}
\end{footnotesize}
\caption{Experiments on four benchmark datasets. For $e$-exponentiated loss functions, we have evaluated $e$-exponentiated softmax loss function with three different value of $e=1.0, 0.75\text{ and }0.60$. The accuracies of other methods have been taken from \cite{MaWHZEXWB18}.\label{tb:comp_results}}
\end{table*}
From the table, we can see that the accuracies obtained by $e$-exponentiated softmax loss with $e=0.6$ are comparable (within the $1\%$ margin) or better $12$ out of $16$ times for methods {\em Backward}, {\em Boot-hard} and $15$ out of $16$ times for method {\em Boot-soft}. However, its performance is relatively worse than that of the methods {\em Forward} and {\em D2L} in which cases the accuracies obtained by $e$-exponentiated loss function are comparable or better only $7$ out of $16$ times. Moreover, in some setting, the accuracy obtained by the two methods is better than that obtained by $e$-exponentiated loss function by a wide margin. However, it should be noted that the scope of our work is to develop better loss functions for the problem and many of the other label correction methods can be used along with our proposed loss functions.

\section{Conclusion}\label{sec:con}
In this paper, we have proposed $e$-exponentiated transformation of loss function. The $e$-exponentiated convex loss functions are almost differentiable, thus can be optimized using gradient descend based algorithm and more robust to outliers. Additionally, using a novel generalization error bound, we have shown that the bound can be tighter for an $e$-exponentiated loss function than that for the corresponding convex loss function in spite of having a much larger Lipschitz constant. Finally, by empirical evaluation, we have shown that the accuracy obtained using $e$-exponentiated loss function can be significantly better than that obtained using the corresponding convex loss function and comparable to the accuracy obtained by some other state of the art methods in the presence of label noise.

\Urlmuskip=0mu plus 1mu\relax
\bibliographystyle{icml2019}
\bibliography{refs.bib}

\section{Proof of Theorem~\ref{thm:2}}\label{app:2}
Before going to the proof of Theorem~\ref{thm:2}, we will state and prove another result which is required for the proof.
\begin{lemma}\label{lm:2}
 Let the expected risk $\mathrm{R}_l(\mbf{w})$ be Lipschitz in small continuous and the corresponding loss function is $L_l$-Lipschitz. Then for $||\mbf{w}_1-\mbf{w}_2||_2\le \epsilon$, $\epsilon>0$, and $\rho>0$
 \begin{align}
  |\hat{\mathrm{R}}_l(\mbf{w}_1)-\hat{\mathrm{R}}_l(\mbf{w}_2)|\le L_{\mathrm{R}_l}(\epsilon)||\mbf{w}_1-\mbf{w}_2||_2+\rho
 \end{align}
is satisfied with probability at least $1-2\;exp\left(-\frac{N\rho^2}{2L_l^2\epsilon^2}\right)$.
\end{lemma}

\begin{proof}
 Since $\mathrm{R}_l(\mbf{w})$ is Lipschitz in small continuous and $||\mbf{w}_1-\mbf{w}_2||_2\le \epsilon$, we have
 \begin{align}
  |\mathrm{R}_l(\mbf{w}_1)-\mathrm{R}_l(\mbf{w}_2)|\le L_{\mathrm{R}_l}(\epsilon)||\mbf{w}_1-\mbf{w}_2||_2 \label{eq:cl_p1}
 \end{align}
 If we let $z_i = l(\mbf{w}_1^T\phi(\mbf{x}_i), y_i)-l(\mbf{w}_2^T\phi(\mbf{x}_i), y_i)$, then we can write
 \begin{smaller}
 \begin{align}
  \E[z] &= \mathrm{R}_l(\mbf{w}_1)-\mathrm{R}_l(\mbf{w}_2), \text{ and }
  \frac{1}{N}\sum_{i=1}^N z_i = \hat{\mathrm{R}}_l(\mbf{w}_1)-\hat{\mathrm{R}}_l(\mbf{w}_2)\nonumber
 \end{align}
 \end{smaller}
 Since $||\mbf{w}_1-\mbf{w}_2||_2\le \epsilon$ and the loss function $l(\cdot, \cdot)$ is $L_l$-Lipschitz function, $|z_i|\le L_l\epsilon$. Using Hoeffding's inequality, we get
 \begin{scriptsize}
 \begin{align}
  &\P\left\{\mathcal{D}_N|\left|\left({\mathrm{R}}(\mbf{w}_1)-{\mathrm{R}}(\mbf{w}_2)\right) - \left(\hat{\mathrm{R}}(\mbf{w}_1;\mathcal{D}_N)-\hat{\mathrm{R}}(\mbf{w}_2;\mathcal{D}_N)\right)\right| \ge \rho\right\} \nonumber\\
  &\le 2\;exp\left(-\frac{N\rho^2}{2L_l^2\epsilon^2}\right).  \label{eq:cl_p2}
 \end{align}
 \end{scriptsize}
Combining Eq.~\eqref{eq:cl_p1} and \eqref{eq:cl_p2}, we complete the proof.
\end{proof}

Now we prove Theorem~\ref{thm:2}.
\begin{proof}[of Theorem~\ref{thm:2}]
 We will mainly follow the proof of \cite{rosasco04}. For simplifying the notation, we ignore the subscript of $\mathcal{D}_N$, $\mathrm{R}_l(\cdot)$ and $\hat{\mathrm{R}}_l(\cdot)$ through out the proof. First of all, by denoting
 \begin{align}
  \Delta_{\mathcal{D}}(\bf{w}) = \mathrm{R}(\bf{w})-\hat{\mathrm{R}}(\bf{w})
 \end{align}
and using Lemma~\ref{lm:2}, we get
\begin{align}
 &\left|\Delta_{\mathcal{D}}(\mbf{w}_1)-\Delta_{\mathcal{D}}(\mbf{w}_2)\right| \nonumber \\
 &\le \left|{\mathrm{R}}(\mbf{w}_1)-{\mathrm{R}}(\mbf{w}_2)\right| + \left|\hat{\mathrm{R}}(\mbf{w}_1;\mathcal{D})-\hat{\mathrm{R}}(\mbf{w}_2;\mathcal{D})\right| \nonumber\\
 &\le 2L_{\mathrm{R}}(\epsilon')||\mbf{w}_1-\mbf{w}_2||_2 + \rho
\end{align}
holds for all $||\mbf{w}_1-\mbf{w}_2||_2\le \epsilon'$ for some $\epsilon'>0$ with probability at least $1-2\;exp\left(-\frac{N\rho^2}{2L_l^2\epsilon'^2}\right)$. Putting $\rho=\frac{L_l\epsilon'\epsilon}{B}$ into the above statement, we get
\begin{align}
 \left|\Delta_{\mathcal{D}}(\mbf{w}_1)-\Delta_{\mathcal{D}}(\mbf{w}_2)\right|\le 2L_{\mathrm{R}}(\epsilon')||\mbf{w}_1-\mbf{w}_2||_2 + \frac{L_l\epsilon'\epsilon}{B} \label{eq:lm_pr_1}
\end{align}
with probability at least $1-2\;exp\left(-\frac{N\epsilon^2}{2B^2}\right)$.
Again, in \cite{rosasco04}, Rosasco et al. have shown that 
\begin{align}
 \P(A) = \P\left(\cup_{i=1}^m\,A_{\mbf{w}_i}\right) \le 2m\;exp\left(-\frac{N\epsilon^2}{2B^2}\right)
\end{align}
where $\mbf{w}_1, \cdots, \mbf{w}_m$ be the $m=C\left(\frac{\epsilon}{2L_{\mathrm{R}}(\epsilon')}\right)$ points such that  the close balls $\mathcal{B}\left(\mbf{w}_i, \frac{\epsilon}{2L_{\mathrm{R}}(\epsilon')}\right)$ with radius $\frac{\epsilon}{2L_{\mathrm{R}}(\epsilon')}$ and center $\mbf{w}_i$ covers the whole set $\mathcal{W}_M=\{\mbf{w}\in \mathbb{R}^d|\;||\mbf{w}||_2\le M\}$ and 
\begin{align}
 A_{\mbf{w}_i} = \left\{\mathcal{D}|\;|\Delta_{\mathcal{D}}(\mbf{w}_i)|\ge \epsilon \right\} \text{ for } i=1, \cdots, m.
\end{align}
When $\epsilon'\ge \frac{\epsilon}{2L_{\mathrm{R}}(\epsilon')}$, for all $\mbf{w}\in \mathcal{W}_M$, there exists some $i\in\{1, \cdots, m\}$ such that $\mbf{w}\in \mathcal{B}\left(\mbf{w}_i, \frac{\epsilon}{2L_{\mathrm{R}}(\epsilon')}\right)$ i.e.
\begin{align}
 ||\mbf{w}-\mbf{w}_i||_2\le \frac{\epsilon}{2L_{\mathrm{R}}(\epsilon')} \label{eq:cover}
\end{align}
\noindent
Note that $\mathcal{D}\in A$ is the dataset for which there exists some $\mbf{w}_i$ whose empirical risk has not converged to its expected risk. Thus, for all $\mathcal{D}\notin A$, we have $|\Delta_{\mathcal{D}}(\mbf{w}_i)|\le \epsilon$ for all $i\in\{1, \cdots, m\}$. Now, combining Eq.~\eqref{eq:lm_pr_1} and \eqref{eq:cover}, we can say that when there exists some $\epsilon'>0$ such that $\epsilon'\ge \frac{\epsilon}{2L_{\mathrm{R}}(\epsilon')}$,
\begin{align}
 \left|\Delta_{\mathcal{D}}(\mbf{w})-\Delta_{\mathcal{D}}(\mbf{w}_i)\right| \le \epsilon + \frac{L_l\epsilon\epsilon'}{B}
\end{align}
holds for all $\mbf{w}\in \mathcal{W}_M$ and some $\mbf{w}_i$ with probability at least $1-2\;exp\left(-\frac{N\epsilon^2}{2B^2}\right)$. Therefore, if there exists an $\epsilon'>0$ such that $\epsilon'\ge \frac{\epsilon}{2L_{\mathrm{R}}(\epsilon')}$,
\begin{align}
 \left|\Delta_{\mathcal{D}}(\mbf{w})\right| \le 2\epsilon + \frac{L_l\epsilon\epsilon'}{B}
\end{align}
hold with probability at least \begin{smaller}$\left(1-2\;exp\left(-\frac{N\epsilon^2}{2B^2}\right)\right)\left(1-2m\;exp\left(-\frac{N\epsilon^2}{2B^2}\right)\right)\ge 1-2(m+1)exp\left(-\frac{N\epsilon^2}{2B^2}\right)=1-2\left(C\left(\frac{\epsilon}{2L_{\mathrm{R}}(\epsilon')}\right)+1\right)exp\left(-\frac{N\epsilon^2}{2B^2}\right)$\end{smaller}. By replacing $\epsilon$ with $\epsilon/2$ and by replacing $\epsilon'$ by $\epsilon$ whenever $\epsilon'> \epsilon$, the statement of the lemma follows.

But, it still remains to show that there always exists an $\epsilon'>0$ such that $\epsilon'\ge \frac{\epsilon}{2L_{\mathrm{R}}(\epsilon')}$. Note that $L_{\mathrm{R}}(\epsilon')$ is a monotonically increasing function of $\epsilon'$. If for some $\epsilon'< \epsilon$, $\epsilon' \ge \epsilon/2L_{\mathrm{R}}(\epsilon')$ holds, we are already done. Else, we have $2\epsilon'L_{\mathrm{R}}(\epsilon')<\epsilon$. Thus, we can increase $2\epsilon'L_{\mathrm{R}}(\epsilon')$ unboundedly by increasing $\epsilon'$, making it larger than $\epsilon$ eventually.
\end{proof}

\end{document}